\documentclass[letterpaper]{article} 
\usepackage[]{aaai2026}  
\usepackage{times}  
\usepackage{helvet}  
\usepackage{courier}  
\usepackage[hyphens]{url}  
\usepackage{graphicx} 
\usepackage{natbib}  
\usepackage{caption} 
\usepackage{algorithm}
\usepackage{algorithmic}
\usepackage{booktabs}

\usepackage{newfloat}
\usepackage{listings}
\DeclareCaptionStyle{ruled}{labelfont=normalfont,labelsep=colon,strut=off} 
\floatstyle{ruled}
\newfloat{listing}{tb}{lst}{}
\floatname{listing}{Listing}
\title{Beyond Static Assumptions: the Predictive Justified Perspective Model for Epistemic Planning (with Supplementary)}
\author {
    Guang Hu\textsuperscript{\rm 1},
    Weijia Li\textsuperscript{\rm 1},
    Yangmengfei Xu\textsuperscript{\rm 1}
}
\affiliations {
    \textsuperscript{\rm 1}The Faculty of Engineering and Information Technology, The University of Melbourne, Australia\\
    \{ghu1, weijia3, yangmengfeix\}@student.unimelb.edu.au
}

\usepackage{bibentry}

\usepackage{amsmath}
\usepackage{amsthm}
\usepackage{listings}
\usepackage{amssymb}
\usepackage{esvect, supertabular}

\newtheoremstyle{normalfont}
  {3pt} 
  {3pt} 
  {\normalfont} 
  {} 
  {\bfseries} 
  {.} 
  { } 
  {} 

\theoremstyle{normalfont}
\usepackage{multirow}

\newcommand{\relationset}{\mathcal{R}}

\newcommand{\override}[1]{\langle #1 \rangle}

\newcommand{\seq}{\vec{s}\,}
\newcommand{\svec}[1]{\vec{#1}\,}

\newcommand{\statespace}{\mathcal{S}}
\newcommand{\seqspace}{\vec{\mathcal{S}}}
\newcommand{\statespacecomplete}{\mathcal{S}_{c}}

\newcommand{\f}{\mathit{f}}
\newcommand{\observation}{\mathit{O}}

\newcommand{\memorization}{\mathit{R}}

\newcommand{\none}{\perp}
\newcommand{\assign}{\!=\!}

\def\unknown{\frac{1}{2}}

\makeatletter
\renewcommand{\ALG@name}{PDDL+ Example} 
\makeatother

\usepackage[many]{tcolorbox}
\definecolor{formalshade}{rgb}{0.95,0.95,1}
\definecolor{darkblue}{rgb}{0.0,0.0,0.5}

\newtcolorbox{myquote}[1][]{%
  colback=formalshade, colframe=darkblue,
  left=4pt, right=4pt, top=4pt, bottom=4pt,
  boxrule=0.5pt,
  sharp corners,
  width=\columnwidth,
  enhanced,
  breakable,
  #1
}

\DeclareUnicodeCharacter{2061}{\textcolor{red}{XXXXXXXXXXXXXXXXXXXXXXXXXXXXXXXXXXXXXXXXXXXXXXX}}

\newcommand{\action}[1]{\textbf{#1}}

\newtheorem{mydefinition}{Definition}
\newtheorem{theorem}{Theorem}

\newtheorem{subdefinition}{Definition}[mydefinition]

\newcommand{\defmathend}{\tag*{\text{$\blacksquare$}}}
\newcommand{\defnormalend}{\hfill$\blacksquare$}

\newtheorem{example}{Example}

\usepackage{listings} 
\usepackage{xcolor}   

\lstdefinelanguage{PDDL}{
    keywords={define, domain, problem, requirements, types, objects, predicates, action, parameters, precondition, effect,goal,init,@,functions,bounds,ranges,rules},
    sensitive=false,
    morecomment=[l]{;},       
    morestring=[b]",          
    alsoletter={.},
}

\begin{document}

\maketitle

\begin{abstract}
\emph{Epistemic Planning} (EP) is a class of planning problems that involve higher-order reasoning about knowledge and beliefs.
While various approaches have been proposed for EP, they generally inherit the static-environment assumption from classical planning.
This limits their applicability in dynamic settings (e.g., a constantly moving object or a falling ball).
To relax this assumption, we extend one of the state-of-the-art models, the \emph{Justified Perspective}~(JP) model, and propose the \emph{Predictive Justified Perspective}~(PJP) model. 
Rather than following the JP model’s assumption that beliefs remain unchanged since the last observation, the PJP model uses past observations to predict changes in variables.
In this paper, we begin by revisiting the JP model and formally defining the PJP model, including the structure and properties of its prediction function.
Following this, we formalize EP problems and illustrate how the PJP model and its semantics can be integrated into this formalism through an example encoding. 
To assess its effectiveness, our approach is evaluated on the most challenging epistemic planning benchmark domain, \emph{Grapevine}, using the simplest complete search algorithm, Breadth-First Search~(BFS).
Experimental results show that the PJP model supports reasoning about (nested) beliefs over changing variables with representational flexibility—a capacity not present in existing epistemic planners—and thus broadens the applicability of EP to dynamic environments.

\end{abstract}

\section{Introduction}
\label{sec:intro}

\emph{Epistemic Planning}~(EP), as an advancement of classical planning, has been developed to enable sound planning based on agents’ knowledge and beliefs.
Currently, all EP approaches follow the assumption that the environment remains static unless explicitly altered by an agent’s actions, an assumption inherited from classical planning.
However, in practical applications, this assumption sometimes fails to hold.
For example, it is unrealistic to treat all pedestrians as static for autonomous driving vehicles operating in busy urban areas.
Accurately modeling pedestrian motion—often represented as a first-order polynomial—is therefore critical, particularly in scenarios involving jaywalking.
Moreover, visibility can be significantly impaired by occlusions such as trucks, increasing risks to both autonomous vehicles and human drivers. 
In such situations, experienced truck drivers may proactively signal surrounding vehicles whose view is obstructed, based on their nested belief that those drivers lack an accurate belief about the jaywalker’s motion model.
Addressing this gap necessitates an epistemic reasoning model capable of accounting for environmental changes, thereby aligning EP frameworks more closely with real-world applications. 
Therefore, to relax this assumption, we build upon an existing epistemic planning framework.

Currently, EP is primarily addressed through three main approaches.
The \emph{Dynamic Epistemic Logic}~(DEL)-based approach was first proposed by \citeauthor{Bolander2011Epistemicsystems}~\shortcite{Bolander2011Epistemicsystems}, and it maintains a Kripke structure~\cite{DBLP:books/mit/FHMV1995} using an event-based model which requires explicit action effects to specify modal logic changes.
The \emph{Knowledge Bases} strategy is another approach that maintains and updates agents' knowledge/belief databases by converting the epistemic planning problems into classical planning problems~\cite{Muise2015PlanningOM,muise2022,DBLP:journals/dm/CooperHMMR19}.
When the scale of the problem increases, both DEL and pre-compilation methods become computationally challenging.
To address this challenge, a \emph{state-based} approach, \emph{Planning with Perspectives}~(PWP)~\cite{Hu2022} leverages external functions and lazy evaluation, which enable offloading the epistemic formula reasoning from the planner, hence improving both efficiency and expressiveness.
Moreover, they define a semantics that reasons based on solely agents' observations, which proved can be done in polynomial time with regard to nesting depth.
However, the PWP approach only handles knowledge (not belief).
A recent continuation study introduced the \emph{Justified Perspectives}~(JP) model to handle the belief~\cite{Hu2023}.
The JP model generates the belief in a way drawing inspiration from two intuitions of human reasoning: humans believe what they see; and, for the parts they could not see, humans believe what they have seen in the past unless they saw evidence to suggest otherwise. 

Compared to other approaches, the state-based approach is more suitable as: 1) it is computationally efficient; 2) it is expressive; and 3) it does not require modeling epistemic logic updates explicitly (as effects of actions). 
Therefore, here, we propose an extension of the JP model to handle dynamic environments and relax the ``static'' assumption. 

In this paper, we first revisit the concept and definition of the JP model in Section~\ref{sec:preliminary}, then propose our new model built upon it. Implementation and experiments are presented to show its effectiveness with epistemic planning instances.

\section{Preliminaries of the JP Model}
\label{sec:preliminary}

The JP model~\cite{Hu2023} was proposed following the intuition~\cite{goldman1979justified}: unless they see evidence to the contrary, agents believe that what they have seen before.
Specifically, when agents infer unobservable variables, if there is no evidence suggesting that these variables have changed, they form justified beliefs by recalling information from their past observations.
``Justified perspectives'' is the term used to represent what agents justifiably believe the sequence of states to be.

A JP \textbf{Signature} is defined as a tuple:
$\Sigma = (Agt,V,\mathbb{D},\mathcal{R})$
where
$Agt$ represents a finite set of agent identifiers containing all agents.
The set $V$ is a finite set of variables such that $Agt\!\subseteq \!V$. 
For each variable $v \!\in \!V$, $D_{v}$ denotes a potentially infinite domain of constant symbols (containing a special value $\none$ to represent none value). 
The domains can be either discrete or continuous, and the overall set of values is given by $\mathbb{D}\!= \!\textstyle\bigcup_{v \in V} \{D_v\}$.
Additionally, $\mathcal{R}$ is a finite set of predicate symbols $r$. 

With the above signature, the state $s$ is defined as a set of assignments that match the variable $v\! \in\! V$ and any value $e\! \in \!D_v$ in the format of $v \!\assign\! e$.
The value of the variable $v$ in a state $s$ can be represented by $s(v)$.
The state space is denoted as $\statespace$ and the complete-state (complete assignments) space is denoted as $\statespacecomplete$.
A global state is a complete state, while a local state might be a partial state.
Besides, the state sequence is denoted as $\seq$ with an index from $0$ to $n$, where $n \!\in \!\mathbb{N}$, and the state with an index of $t$ from this sequence can be represented by $\seq[t]$, and its value for variable $v$ can be represented as $\seq[t](v)$.
Similarly, the sequence space is denoted as $\vec{\statespace}$ and the complete-sequence space is $\vec{\statespacecomplete}$.
Besides, the following override function is provided to update states.

\begin{mydefinition}[State Override Function]
Given a state $s$ and $s'$, the state override function \( \langle \rangle: \statespace \times \statespace \to \statespace \) to override $s'$ with $s$ is defined as follows: 
\[
s'\langle s \rangle = s \cup \{ v \!=\! s'(v) \mid v \in s' \text{ and } v \notin s \} \defmathend
\]
\end{mydefinition}
The above function overwrites state $s'$ with state $s$, preserving assignments from $s$ and maintaining those variables found only in $s'$ but not in $s$.
In addition, the state override function can also be applied to a sequence of states, where $s' \override{\seq} = [s' \override{\seq[0]},\dots,s' \override{\seq[n]}]$.


\begin{mydefinition}[Language]\label{def:language}
The language of Knowledge and Belief, $L_{K\!B}(\Sigma)$, of the JP model is defined by the grammar:
\[
\varphi ::= r(V_r) \mid \neg \varphi \mid \varphi \land \varphi \mid S_i v \mid S_i \varphi \mid K_i \varphi \mid B_i \varphi,
\]
where $r \in \relationset$, $V_r \subseteq V$ are the terms of $r$, and $r(V_r)$ forms a predicate ; $v\in V$, and $i \in Agt$. \defnormalend
\end{mydefinition}

$S_i v$ means agent $i$ sees variable $v$.
Similarly, $S_i \varphi$, $K_i \varphi$, and $B_i \varphi$ mean agent $i$ sees, knows, and believes formula $\varphi$, respectively.
The difference between knowledge and belief is that the knowledge of $\varphi$ is consistent with the actual world while the belief may not be the truth. 
For illustration, in this paper, we denote the set of all predicates as $\mathcal{P}$.
\begin{mydefinition}[JP Model]\label{def:JPM}
With the above signature and language, the JP \textbf{model} $M$ is defined as:
\[
M = (Agt, V, \mathbb{D}, \pi, O_1, \ldots, O_{|Agt|}),
\]
where:
$Agt$, $V$ and $\mathcal{D}$ are from the signature $\Sigma$;
$\pi$ is an interpretation function, $\pi : \mathcal{S} \times \mathcal{P} \rightarrow \{true,false\}$, that determines whether the predicate $r(V_r)$ is true in $s$.
Finally, $O_1, \ldots, O_{|Agt|}$ are observation functions defined below. \defnormalend
\end{mydefinition}
Hence, a special state $s_\none$ can be used to represent a state in which all variables are $\none$ ($s_\none = \{v \assign \none \mid v \in V\}$).
\begin{mydefinition}[Observation Function]
\label{def:observation}
An observation function for agent~$i$, $\observation_i: \statespace \rightarrow \statespace$, is a function that takes a state and returns a subset of that state, representing the part of the state visible to agent~$i$. 
The following properties must hold for an observation function $\observation_i$ for all $i \in Agt$ and $s \in S$:
\begin{equation*}
    \begin{aligned}
            1. &\ \observation_i(s) \subseteq s , & \text{(Contraction)}\\
            2. &\ \observation_i(s) = \observation_i(\observation_i(s)), &  \text{(Idempotence)}\\
            3. &\ \text{ If } s \subseteq s', \text{ then } \observation_i(s) \subseteq \observation_i(s'), &  \text{(Monotonicity)}\\
    \end{aligned}\defmathend
\end{equation*} 
\end{mydefinition}
The idea is to reason about the agents' epistemic formulae based on what they can observe from their own local states. 
This is accomplished by defining an observation function for each agent $i$.
Although the JP model reasons with sequences, its original authors did not specify the observation function to handle sequences.
In this paper, we denote $\observation_i(\seq) = [\observation_i(\seq[0]),\dots,\observation_i(\seq[n])]$ (\textit{i.e.} $\observation_i(\seq[t])\equiv \observation_i(\seq)[t]$).

Then, the \textbf{Retrieval function} is introduced to retrieve the most recent value of the variable $v$.
\begin{mydefinition}[Retrieval Function]
\label{def:jp:R}
Given a sequence of states  $\seq$, a timestamp $m$, and a variable $v$, the retrieval function, $\memorization: \seqspace \times \mathbb{Z} \times V \rightarrow \mathbb{D}$, is defined as:
    \[
    \memorization(\seq,m,v) = 
\begin{cases} 
\seq[\max(\text{LT})](v) & \text{if } \text{LT} \neq \{\} \\
\seq[\min(\text{RT})](v) & \text{else if } \text{RT} \neq \{\} \\
\none & \text{otherwise}
\end{cases}
    \]
    where LT $= \{j \mid v \!\in \!\seq[j] \ \land \ \seq[j](v) \!\neq \none \ \land \ j \!\leq m \}$ and \\RT $= \{j \mid v \!\in \!\seq[j] \ \land \ \seq[j](v) \!\neq \none \ \land \ m \!< \!j \!\leq \! |\seq|\}$. \defnormalend
\end{mydefinition}
If the variable $v$ appears in any previous state (or the current state) within the given justified perspective, we assume that \( v \) has remained unchanged since its most recent timestamp, \( \max(\text{LT}) \), which is the latest time \( v \) was observed in the given perspective (at or before the given timestamp \( m \)).  
If \( v \) has never been observed in the given justified perspective before, we assume that \( v \) has remained unchanged since the earliest timestamp \( \min(\text{RT}) \), which is the first time \( v \)'s value appeared in the given perspective after the given timestamp \( m \).  
Otherwise, \( v \) is considered \( \none \), as the given perspective does not contain \( v \) at all.

Then the \textbf{justified perspective function} can construct the agents' perspectives under the input state sequence.
\begin{mydefinition}[Justified Perspective Function]
\label{def:jpf}
Given the input state sequence $\seq$ as $[s_0,\dots,s_n]$, a \emph{Justified Perspective} (JP) function for agent~$i$, $\f_i: \vec{\statespacecomplete} \rightarrow \vec{\statespacecomplete}$, is defined as follows:
    \[
    \f_i([s_0,\dots,s_n])=[s'_0,\dots,s'_n] 
    \]
    where for $t \in [0, n]$ and $v \in V$:
    \[
        \begin{aligned}
            lt_v & = \max(\{j \mid v \in \observation_i(s_j) \land j \leq t \} \cup \{ -1\}) \text{,} & (1)\\
            e    & = \memorization([s_0,\dots,s_t],lt_v,v), & (2)\\
            s''_t & = \{v \assign e \mid s_t(v)=e \lor v \notin \observation_i(s_t\langle \{v \assign e\} \rangle) \}, & (3)\\
            s'_t  & = s_\none \langle s''_t \rangle. & (4)
        \end{aligned} \defmathend
    \]
\end{mydefinition}
Function \( \f_i \) takes a sequence of states (either the global state sequence or a justified perspective from another agent) and returns the sequence of local states that agent \( i \) believes.  
Line (4) overrides state $s_\none$ with state $s''_t$, which ensures that the output is a complete state sequence, where any missing variables are filled with a placeholder value \(\none\).
\( lt_v \) represents the last timestamp at which agent \( i \) observed the variable \( v \), where ``-1'' is added to ensure it returns an integer.
Using Function \( R \), as shown in Line (2), the value of $v$ that agent $i$ observed (or should have observed) at the relevant timestamp is retrieved.
Subsequently, Line (3) constructs a justified state at timestamp \( t \), while the condition \( v \! \notin\! \observation_i(s_t\langle \{v\! \assign\! e\} \rangle) \) ensures that agent \( i \)'s ``memory'' remains consistent with its current observations at timestamp \( t \).

Then, a ternary semantics is proposed for JP, which employs three truth values: 0 (false), 1 (true), and $\frac{1}{2}$ (unknown). 
\begin{mydefinition}[Ternary semantics]
\label{def:jp:ternary}
Given a state sequence $\seq$ any epistemic formula in Language $L_{K\!B}(\Sigma)$, the ternary function $T$ is defined as, omitting model $M$ for readability:
\begin{supertabular}{@{}ll@{~}l@{~~}l@{~~}r}
  (a) & $T[\seq, r(V_r)]$ & $=$ & 1 if $\pi(\seq[n], r(V_r)) = true$;\\
             &                          &     & 0 else if $\pi(\seq[n], r(V_r)) = false$;\\
             &                          &     & $\unknown$ otherwise\\[1mm]
            (b) & $T[\seq, \varphi \land \psi]$ & $=$ & $\min(T[\seq, \varphi], T[\seq, \psi])$\\[1mm]
            (c) & $T[\seq, \neg \varphi]$    & $=$ & $1 - T[\seq, \varphi]$\\[1mm]
            (d) & $T[\seq, S_i v]$           & $=$ & $\unknown$ if $ v \notin \seq[n]$ or $ i \notin \seq[n]$\\[0.5mm]
             &                          &   & $0$ else if $v \notin \observation_i(\seq[n])$\\
             &                          &   & $1$ otherwise     \\[1mm]
            (e) & $T[\seq, S_i \varphi]$  & $=$ & $\unknown$ if $T[\seq,\varphi] = \unknown$ or $ i \notin \seq[n]$;\\
             &                          &   & $0$ else if $T[\observation_i(\seq), \varphi] = \unknown$;\\
             &                          &   & $1$ otherwise\\[1mm]
            (f) & $T[\seq, K_i \varphi]$ & = & $ T[\seq, \varphi \land S_i\varphi]$\\[1mm]
            (g) & $T[\seq, B_i \varphi]$ & = & $ T[\f_i(s_{\none} \override{\seq}), \varphi]$ \\[1mm]
            \multicolumn{4}{l}{where $n$ is the last (current) timestamp of $\seq$.} & $\blacksquare$\\
\end{supertabular}
\end{mydefinition}

That is, the semantics of the belief is handled by reasoning with agents' generated justified perspectives.
This ternary function $T$ returns $1$ or $0$ when the formula is evaluated as \emph{true} or \emph{false}, respectively.
When there is not enough valid value in the given local state $\seq[n]$ to evaluate a predicate $r(V_r)$ ($\exists v \in V_r, \seq[n](v) = \none$), or the relevant variables (nesting agent's identifier $i$ or targeting variable $v$) are not in the given local state ($i \notin \seq[n]$ or $v \notin \seq[n]$), the ternary function returns $\unknown$, indicating that the truth value of the formula cannot be determined.

In addition, \citeauthor{Hu2023}~\shortcite{Hu2023} proved that their ternary semantics are in \textbf{KD45} axiomatic system and the evaluation of it is in polynomial time. 
At last, they provided experiments to demonstrate its efficiency on EP benchmarks with comparisons against the current state-of-the-art planner PDKB planner~\cite{muise2022}.
\section{Predictive Justified Perspective Model}
\label{sec:model}

As mentioned in Section~\ref{sec:intro}, the ``static'' assumption in the existing EP works (including the JP model) is not practical in many applications.
For example, the ``static'' variable cannot capture the position change of a falling ball, as its position change is not an effect of an agent's action but due to the environmental effect (gravity).
To differentiate from the concept of the ``static'' variable in AI planning, the variables in our model are named as \emph{processual}\footnote{The idea of the process is from PDDL+~\cite{DBLP:journals/jair/FoxL06}.} variable which are state variables whose values can evolve as a result of external environmental processes, independently of the agent's actions. 
To describe a set of processual variables, the processual variable model is denoted as $\Omega$, and defined below.
\begin{mydefinition}[Processual Variable Model]\label{def:omega}
Given $\mathcal{T}$ is a set that includes all processual variable types, $\Omega$ is defined as:
\[
\Omega = (V,\mathcal{T},type,\eta),
\]
where: $type: V \rightarrow \mathcal{T}$, $\eta: V  \rightarrow \mathbb{R}^*, * \in \mathbb{N}$. \defnormalend
\end{mydefinition}
To indicate the changing rules, for each processual variable $v \in V$, the type and coefficients (or parameters) are defined as $type(v)$ and $\eta(v)$. 
A ``static'' variable is modeled by the new framework as a processual variable with a special type \emph{static}.
It is considered as a base case ($*\!\!=\!\!0$) in which $\eta(v)\! =\! \{()\}$.
Unlike classical planning variables, which are modified exclusively by the agent’s action effects, processual variables are updated according to not only the agent's action effects but also exogenous transition rules or continuous dynamics reflecting changes in the environment.
In the remainder of this paper, the term ``variable'' refers to a processual variable unless stated otherwise.

Thus, with the introduction of the processual variable, the PJP model --- following the same signature and language as the JP model introduced in Section~\ref{sec:preliminary} --- can be defined.
\begin{mydefinition}[Model]
The PJP model $M$ is defined as:
\[
M = (Agt, \Omega,\mathbb{D}, \pi, O_1, \ldots, O_{|Agt|}, P\!R),
\]
where $P\!R$ is a set of predictive retrieval functions $pr_{type(v)}$ (Definition~\ref{def:pr_function}); $\Omega$ is the model defined in Definition~\ref{def:omega}; and the rest are adopted from the JP model (Definition~\ref{def:JPM}). \defnormalend
\end{mydefinition}
Incorporating processual variables, the PJP model generates predictive justified perspectives through the \emph{Predictive Justified Perspective Function} (Definition~\ref{def:pjpf}) relying on the \emph{Predictive Retrieval Function}. An abstract definition of the predictive retrieval function, $pr_{type(v)} \in P\!R$, where $type(v) \in \mathcal{T}$ is provided in Definition~\ref{def:pr_function}.
This provides flexibility and expressiveness for the modeler to model any variable with a known type.

\begin{mydefinition}[Predictive Retrieval Function]
    \label{def:pr_function}
    A predictive retrieval function $pr_{type(v)}: \vec{S} \times \mathbb{N} \times V \rightarrow \mathbb{D}$ takes the input of a state sequence $\seq$,  a timestamp $t$ and a variable $v$, and outputs the predicted value of $v$ at $t$ based on $\seq$, where $type(v) \in \mathcal{T}$ is the processual variable $v$'s type.
    The predictive retrieval function $pr_{type(v)}$ should satisfy the following properties.
    Given the input state sequence $\seq =  [s_0, \dots, s_n]$, its prediction $\vec{p}$ is defined as $\vec{p} = [p_0, \dots, p_n]$, where for $t' \in [0,n]$, $p_{t'} = \{ v \assign pr_{type(v)}(\seq,t',v) | v \in V\}$.
\begin{itemize}
    \item \textbf{Preserving Consistency} [\textit{Compulsory}]:
    \[
        \forall v\!\in\!V, \forall t\! <\! |\seq|, v \!\in\! \seq[t] \Rightarrow \svec{p}[t](v) = \seq[t](v)
    \]
    \item \textbf{Recursive Consistency} [\textit{Compulsory}]:
    Let $\vec{W}$ be a subset of the sequence space $\vec{\statespace}$, such that for all $\vec{w} \!\in\! \vec{W}$: $|\vec{w}|\!=\!|\seq|$ and $\forall t\! < \! |\seq| \Rightarrow \seq[t] \subseteq \vec{w}[t] \subseteq \svec{p}[t]$.
    \[
        \forall v \!\in\! V, \forall t \!<\! |\seq|,\forall \vec{w} \!\in\! \vec{W} \Rightarrow pr_{type(v)}(\vec{w},t,v) = \svec{p}[t](v) 
    \]
    \item \textbf{Reconstructive Consistency} [\textit{Optional}]:
    \[
        \forall i \in Agt,\exists \vec{w} \in \vec{\statespace}, \observation_i(\vec{w}) \assign \seq \Rightarrow \seq \assign \observation_i(\svec{p}) \defmathend
    \]
\end{itemize} 
\end{mydefinition}

The prediction function $pr_{type(v)}$ estimates the value of variable $v$ at timestamp $t$ by deriving $v$'s changing pattern based on the given state sequence $\seq$.
A valid predictive retrieval function must be preserving consistent and recursively consistent.
\textbf{Preserving Consistency} ensures the predicted value of the variable is consistent with its input.
This gives $\forall t\! <\! |\seq| \Rightarrow \seq[t] \subseteq \svec{p}[t]$, where $\vec{p}$ is constructed in the above definition.
\textbf{Recursive Consistency} 
ensures that the predictive retrieval function yields stable results when the input state sequence is the original input state sequence ($\seq$) with extra values from its prediction ($\svec{p}$).
This consistency condition requires that for all sequences in the set of sequences $\vec{W}$ (which is the same as the original input sequence with extra assignments from its prediction), applying the predictive retrieval function to $\vec{w}$ yields the same predicted value as in $\vec{p}$.
A simple example is for a sequence that only contains one variable $v'$ with values of $[1,\none,3,\none]$.
By applying $pr_{type(v')}([1,\none,3,\none],t',v')$ for all timestamps $t' < 3$, we have predicted values of $v$ as $[1,2,3,4]$.
Then, for an input sequence that is the same as the original sequence but with extra values from its prediction, such as $[1,2,3,\none]$, the predicted value $pr_{type(v')}([1,2,3,\none],t',v')$ should stay the same for all timestamps $t' < 3$ ($[1,2,3,4]$).
The optional \textbf{Reconstructive Consistency} ensures that for all agents, the predictive function is consistent with their own observation function.
That is, when using the agent's observation as the input, the predicted values should not change the agent's observation.
This property ensures agents' beliefs are justifiable ($B_i K_i \varphi \Rightarrow K_i \varphi$) following the JP model.

Then, following Definition~\ref{def:pr_function}, we give the definition of the base case of the predictive retrieval function, $pr_{static}$.
The function $pr_{static}$ is domain-independent due to the ``static'' assumption of classical planning.
While for other $type(v) \in \mathcal{T}$, $pr_{type(v)}$ is domain-dependent, and they need to be defined by the modeler. 

\begin{subdefinition}[Static Predictive Retrieval Function] 
Given a sequence of states $\seq$, a timestamp $t$, and a variable $v$, the retrieval function, the static predictive retrieval function $pr_{static} \in P\!R$ , is defined as:
\label{def:pr_static}

    $pr_{static}(\seq,t,v) = \memorization(\seq,t,v)$, where $R$ in Definition~\ref{def:jp:R} \defnormalend
\end{subdefinition}

The behavior of $pr_{static}$ is semantically equivalent to the retrieval function $R$ defined in the JP model (following the static assumption). 
Hence, the result of $pr_{static}(\seq,t,v)$ is the same as $R(\seq,t,v)$, ensuring compatibility between the predictive reasoning in the PJP model and the belief reasoning in JP for static variables.
According to the theorem ($\f_i(s)=\f_i(\f_i(s))$) and its proof of the JP model~\cite{Hu2023}, $pr_{static}$ meets all properties in Definition~\ref{def:pr_function}.

Then, it is up to the modeler on how to design $P\!R$ functions for other processual variable types.
With all $P\!R$ functions that follow Definition~\ref{def:pr_function}, we can now give our definition to generate justified perspectives with prediction.

\begin{mydefinition}[Predictive Justified Perspective Function]
\label{def:pjpf}
 Given the input state sequence $[s_0,\dots,s_n]$, a \emph{Predictive Justified Perspective}~(PJP) function for agent~$i$, $\f_i: \vec{\statespace} \rightarrow \vec{\statespacecomplete}$, is defined as follows:
    \[
        \f_i([s_0,\dots,s_n])=[s'_0,\dots,s'_n] 
    \]
    where for $t \in [0, n]$ and $v \in V$:
    \[
        \begin{aligned}
            e     & = pr_{type(v)}(\observation_i([s_0,\dots,s_n]),t,v), & (1)\\
            s''_t & = \{v \assign e \mid v \!\in \!\observation_i(s_t) \lor v\! \notin\! \observation_i(s_t\langle \{v \assign e\} \rangle) \}, & (2)\\
            s'_t  & = s_\none \langle s''_t \rangle. & (3)
        \end{aligned} \defmathend
    \]
\end{mydefinition}

The PJP Function $f_i$ generates local predictive justified perspective for agent~$i$. 
A value $e$ of the variable $v$ is retrieved by its corresponding predictive retrieval $pr_{type(v)}$ in Line (1).
If the type is \emph{static}, the value of $v$ is the same as it is returned by the original retrieval function in the JP model (Definition~\ref{def:jp:R}). 
$v\! \in\! \observation_i(s_t)$ in Line (2) represents agent $i$ sees $v$ at $t$, and thus, the observed $v$ and its value $e$ will be in the output perspective.
When the agent $i$ cannot see $v$ at $t$, by assigning $e$ to $v$ at $t$ (represented as $s_t \override{\{v \assign e\}}$), agent $i$ still cannot see $v$ (represented as $v \!\in \!s_t \override{\{v \assign e\}}$), the predicted value $e$ will be assigned to $v$. 
It means when the agent $i$ cannot see $v$ at $t$, the predicted value $e$ does not cause the conflict of the fact that agent $i$ cannot see $v$ at $t$.
Line~(3) ensures the generated perspectives are complete-state sequences.

\subsection{An Example Using First-Order Polynomials}

To demonstrate the usage of the Predictive Retrieval function and PJP function, we introduce an example using the first-order polynomial as the processual variable types.
Firstly, we extend a classical epistemic planning benchmark, \emph{Grapevine}~\cite{Muise2015PlanningOM}, by incorporating processual variables.
\begin{example}[Grapevine]
\label{example:lie}
There are 3 agents, $a$, $b$, and $c$, each with their own secret: $sa$, $sb$, and $sc$, and there are two rooms, $rm_1$ and $rm_2$.  
The actions include \action{move}, \action{share}, and \action{lie}.  
The shared value of the secret, using $sa$ as an example, is represented as $ssa$, which equals $tsa$ when $a$ is sharing and $lsa$ when $a$ is lying.  
Initially, all agents are in $rm_1$.  
The following action sequence is performed:\\
$[$\action{share}($a$), \action{stop}, \action{share}($a$), \action{stop}, \action{move}($c,rm_2$), \action{lie}($a$), \action{stop}$]$
\end{example}


Here, we provide our definition\footnote{Note: this is just one valid (reasonable) definition of predictive retrieval function for the first-order polynomial. } of $pr_{1st\_poly}$ for Example~\ref{example:lie}.
Since the shared value of the secret ($ssa$) could be deceptive (by lying), the values in an agent's observation could come from different polynomials ($tsa$ or $lsa$).
Therefore, we use segmented prediction in this example.

\addtocounter{mydefinition}{-1}
\addtocounter{subdefinition}{1}
\begin{subdefinition}[First-order Polynomial Predictive Retrieval Function]
\label{def:pr_1st_poly}
    Given AT$= \{ j\! \mid \!v\! \in \!\seq[j] \land \seq[j](v) \neq \none \}$,
            LT $=\! \{j \!\mid\! j\!\in\! \text{AT} \!\land j\! \leq\! t\}$, and RT$=\! \{j\! \mid \!j\in\! \text{AT}\! \land\! j\! >\! t\}$,
    the predictive retrieval function for the first-order polynomial, $pr_{1st\_poly}(\vec{s},t,v)$, can be defined as:
    \[\begin{aligned}
    pr&_{1st\_poly}(\vec{s},t,v)= \\
         &\begin{cases} 
            \none & \text{if } |\text{AT}| = 0\\
            \seq[\max(\text{AT})](v) &\text{if } |\text{AT}| = 1\\
            \frac{t - t_1}{t_1 - t_2}\left(\seq[t_1](v) - \seq[t_2](v)\right)+\seq[t_1](v) & \text{if } |\text{AT}| \geq 2
            \end{cases}
    \end{aligned}
      \] 
    where
    \begin{equation*}
            \begin{cases} 
                \text{if RT} = \emptyset, & t_2 = \max(\text{LT}), \, t_1 = \max(\text{LT} \setminus \{t_2\}) \\ 
                \text{if LT} = \emptyset, & t_1 = \min(\text{RT}), \, t_2 = \min(\text{RT} \setminus \{t_1\}) \\ 
                \text{otherwise}, & t_1 = \max(\text{LT}), \, t_2 = \min(\text{RT})
            \end{cases} \defmathend
    \end{equation*}
\end{subdefinition}

\addtocounter{mydefinition}{1}
The predictive retrieval function for first-order polynomial identifies the most recent two timestamps in terms of the given timestamp $t$ that agent~$i$ observes $v$, denoted as $t_1$ and $t_2$ ($t_1 \!\neq\! t_2$, $\{t_j \mid j \in \text{AT}\}$), such that: 
$t_1 \!< \!t_2 \!\leq \!t$, if timestamp $t$ is at or after the agent's latest observation of $v$;
$t \!<\! t_1 \!<\! t_2$, if timestamp $t$ is before the agent's first observation of $v$;
$t_1 \! \leq \! t \!< \!t_2$, otherwise.
For example, if the last observation is before or at the timestamp $t$, $t_2$ will be the timestamp of the last observation and $t_1$ will be the timestamp of the second-last observation where $t_1\! =\! \max(\text{LT} \setminus \{t_2\})$.

\begin{table}[H]
    \centering
    \resizebox{\linewidth}{!}{
        \begin{tabular}{@{~}cc@{~}c@{~}c@{~}c@{~}c@{~}c@{~}c@{~}c@{~}}
            \toprule
            \textbf{Timestamp}& \textbf{$0$} & \textbf{$1$} & \textbf{$2$} & \textbf{$3$} & \textbf{$4$} & \textbf{$5$} & \textbf{$6$} & \textbf{$7$} \\
            \midrule
            \textbf{Action} &  &\action{share}($a$) & \action{stop} & \action{share}($a$) & \action{stop} & \action{move}($c$,$rm_2$) & \action{lie}($a$) & \action{stop} \\
            \midrule
            $tsa$ & 3 & 4 & 5 & 6 & 7 & 8 & 9 & 10 \\
            $lsa$ & 1 & 2 & 3 & 4 & 5 & 6 & 7 & 8 \\
            $ssa$ & \_ & 4 & \_ & 6 & \_ & \_ & 7 & \_ \\
            \hline
            $O_b(\vec{s})(ssa)$ & \_ & 4 & \_ & 6 & \_ & \_ & 7 & \_ \\
            $O_c(\vec{s})(ssa)$ & \_ & 4 & \_ & 6 & \_ & \_ & \_ & \_ \\
            $f_b(\vec{s})(ssa)$ & 3 & 4 & 5 & 6 & 6.33 & 6.67 & 7 & 7.33 \\
            $f_c(\vec{s})(ssa)$ & 3 & 4 & 5 & 6 & 7 & 8 & 9 & 10 \\
            $f_b(f_a(\vec{s}))(ssa)$ & 3 & 4 & 5 & 6 & 6.33 & 6.67 & 7 & 7.33 \\
            $f_b(f_c(\vec{s}))(ssa)$ & 3 & 4 & 5 & 6 & 7 & 8 & 9 & 10 \\
            $f_c(f_a(\vec{s}))(ssa)$ & 3 & 4 & 5 & 6 & 7 & 8 & 9 & 10 \\
            $f_c(f_b(\vec{s}))(ssa)$ & 3 & 4 & 5 & 6 & 7 & 8 & 9 & 10 \\
            \bottomrule
        \end{tabular}
    }
    \caption{Observation and reasoning of $ssa$ in Example~\ref{example:lie}.}
    \label{tab:ssa_pred_table}
\end{table}

The values of $ssa$ (shared value of $sa$) in each agent's observation (what they saw) and predictive justified perspective (what they believe), which are generated by iteratively applying observation functions and PJP functions, are shown in Table~\ref{tab:ssa_pred_table}.
Before timestamp 5, all agents are in $rm_1$, which results in them having the same (nested) observations and (nested) beliefs.
However, after action \action{move}($c,rm_2$), $c$ is no longer in $rm_1$, and won't be able to ``see'' the following $ssa$ value.
This leads to different beliefs between agent $b$ ($f_b(\seq)$) and agent $c$ ($f_c(\seq)$) about the value of $ssa$, and agent $b$ is able to reason about agent $c$ beliefs ($f_b(f_c(\seq))$).
For example, when generating $s''_4$ (timestamp $4$ in Definition~\ref{def:pjpf}) in $f_b(\seq)$, $s_4''(ssa) = pr_{1st\_poly}(\observation_b(\seq),4,ssa)=6.33$.
This is calculated (Definition~\ref{def:pr_1st_poly}) by identifying $t_1$ and $t_2$ ($3$ and $6$), and using $ssa$'s values ($6$ and $7$) to get its value for timestamp $4$.
But for $s_4''$ in $f_c(\seq)$,  $s_4''(ssa) = pr_{1st\_poly}(\observation_c(\seq),4,ssa)$ would be $7$ by identifying $t_1$ and $t_2$ as $1$ and $3$, and using $ssa$'s values ($4$ and $6$) to predict its value at timestamp $4$.
This is because of the difference between $O_b(\seq)$ and $O_c(\seq)$.
More challengingly, agent $b$'s belief of agent $c$'s belief is generated by $O_c(f_b(\seq))$ (same as $O_c(\seq)$ in this example), resulting in agent $b$ believing that: 1) $ssa \assign 7.33$; and, agent $c$'s belief of $ssa$ is $7$.
%

\subsection{Semantics and Axiomatic Validity}
\label{sec:validity}
By constructing predictive justified perspectives (Definition~\ref{def:pjpf}), the reasoning of the epistemic formulae in $L_{KB}$ follows the ternary semantics defined in the JP model (Definition~\ref{def:jp:ternary}), except that the perspective function is changed from the justified perspective function (Definition~\ref{def:jpf}) to the PJP function (Definition~\ref{def:pjpf}) defined in this paper.

To examine the soundness of the PJP model, we prove our ternary semantics satisfy \textbf{KD45} axioms given all predictive retrieval functions are preserving consistent and recursive consistent.
Moreover, we show that if all predictive retrieval functions are also reconstructive consistent, then we have $f_i(\seq) \assign f_i(f_i(\seq))$, which ensures agents' beliefs are justifiable (aligned with the JP model~\cite{Hu2023}). 
In terms of computational efficiency, the complexity of the ternary semantics for our PJP model is in polynomial if $O$ and $P\!R$ are in polynomial (same as the JP model).
We include the theorem, proofs, and detailed explanation of the above findings in the supplementary material, as they are tangential to the main contributions (Planning community) and may be of interest primarily to a specialized subset of readers (Epistemic Logic community).

\section{Implementation}
\label{sec:planning}
In this section, we explain how the EP problems with a non-static assumption are modeled and encoded, integrating the PJP model as an epistemic logic evaluator.

\subsection{Problem Formalization}
As a state-based approach in EP, an epistemic planning problem with signature $\Sigma$ can be represented by an epistemic reasoning model $M$ and an epistemic planning instance $P$~\cite{Hu2022,Hu2023}.
The planning instance can be represented as $P=(Agt,\Omega,\mathbb{D},E,O_F,I,G_F,F)$, where: $Agt$ and $\mathbb{D}$ are from signature $\Sigma$; $\Omega$ is from $M$ (Definition~\ref{def:omega}); $E$ is the set of all involved epistemic formulae ($E \subseteq L_{K\!B}(\Sigma)$); $I$ is the initial state same as in classical planning; $O_F$ and $G_F$ are the operators (actions) and goal conditions involving epistemic formulae (by external function); and, $F$ is the external functions ($@jp: \vec{\statespace} \times E \to \mathbb{D}$ or $@epi: \vec{\statespace} \times E \to \{0,\unknown,1\}$) that implement ternary semantics (Definition~\ref{def:jp:ternary}) with the PJP model.
The details of $\Omega$, $O_F$, $G_F$, and $F$ are provided in the following sections.

\subsection{Encoding}
\label{sec:planning:encoding}
The EP instances are represented by PDDL+~\cite{DBLP:journals/jair/FoxL06} and integrate the idea of external functions from F-STRIPS~\cite{Geffner2000}.
Similar to PDDL+, $\Omega$ is modeled as processes (we name them rules, as shown below).

\begin{lstlisting}
(:rules
  (static (agent_loc a) [] [])
  (1st_poly (secret sa) [1.0,2.0] [,])
  ...
)
\end{lstlisting}

The type static (Line 2) has no coefficient, while the type first-order polynomial has two coefficients. 
The second square bracket represents the agent's initial belief of the coefficients.
That is, Line 3 indicates that the change is $sa = t + 2$ (where $t$ is the timestamp), and no coefficients are known in advance.

The epistemic formulae appear in the action precondition and goal as external functions, while they could also be included in the action effects.
Two types of external functions, \emph{@epi} and \emph{@jp} are provided.
External function type \emph{@epi} is for evaluating a normal epistemic formula, for example, $T[\seq,\neg B_b \ ssa\assign 6]=1$, as shown below:
\begin{lstlisting}
(= 
  (@epi ("-b[b]") (= (shared sa) 6)) 
      epi.true)
\end{lstlisting}



The external function type \emph{@jp} evaluates the agent's local state variable directly on the agent's justified perspective.
An example action, \action{share\_others\_secret}, is provided in PDDL+ Example~\ref{alg:action}.
It allows an agent to share another agent's secret (\texttt{?s}) with others if the agent has a not-none belief in this secret.
The preconditions ensure that the secret \texttt{?s} is not owned by this agent \texttt{?a} (Line~4) and there is no one sharing anything at the current timestamp (Line~5).
In addition, in Line~6-8, the query \texttt{"b[?a]"} represents whose justified perspective is evaluated on, while the second argument \texttt{(shared ?s)} specifies which variable it is.
This precondition holds if agent \texttt{?a}'s perspective has a not-none value of \texttt{(shared ?s)}.
The effects specify the secret \texttt{?s} being shared (Line~16) at agent \texttt{?a}'s location (Line~11-12), while Line~13-15 represents that the shared value of the secret is the same as the value that the sharing agent \texttt{?a} believes.
This lifted encoding allows the solver to perform actions with epistemic effects lazily without grounding, which is not found in any other epistemic planning approaches.

\begin{algorithm}
\caption{share\_others\_secret}
\label{alg:action}
\begin{lstlisting}
(:action share_others_secret
  :parameters(?a - agent, ?s - secret)
  :precondition(
    (= (own ?a ?s) 0)
    (= (sharing) 0)
    (!= 
      (@jp ("b[?a]") (shared ?s)) 
          jp.none)
  )
  :effect(
    (assign 
      (shared_loc ?s) (agent_loc ?a))
    (assign 
      (shared ?s) 
          (@jp ("b[?a]") (shared ?s))) 
    (assign (sharing) 1)
  )
)
\end{lstlisting}
\end{algorithm}

\subsection{External Functions}
\label{sec:planning:external}
The external functions take the current search path (state sequence $\seq \in \vec{\statespace}$) and the epistemic formula $\varphi \in E$ as input and return a ternary value or a value that the agent believes (in their perspectives) as output depending on the external function types ($@epi$ or $@jp$) as explained in Section~\ref{sec:planning:encoding}.
The $@jp$ functions are simply implementing the PJP function (Definition~\ref{def:pjpf}) iteratively following the encoding, while the $@epi$ functions implement the ternary semantics (Definition~\ref{def:jp:ternary}) with the PJP model. 

\subsubsection{Predictive Retrieval Functions}
The PJP model possesses strong expressiveness and high extensibility, which stem from the design of the predictive retrieval function $pr_{type(.)}\in P\!R$. 
We implemented different $P\!R$ functions for 5 common processual variable types: first-order polynomial, second-order polynomial, power function, sine function, and static.
For different types, different numbers of observations are needed to deduce their coefficients and predict their values.
For all $pr_{type(\cdot)}$ functions used in this paper, if the observations are insufficient, the value is returned in the same way as $pr_{static}$ as a fallback.
This is just one valid design of the predictive retrieval function satisfying the properties in Definition~\ref{def:pr_function} for the insufficient inputs.

Beyond these examples, the model’s flexibility is further demonstrated by its support for incorporating alternative prediction functions (Definition~\ref{def:pr_function}). 
This allows for the integration of learning-based methods such as linear regression, support vector machines, and neural networks.

\begin{table*}[ht]
    \centering
    \resizebox{1\textwidth}{!}{
        \begin{tabular}{ccccccccccccc}
            \toprule
            & $|gen|$ & $|seg|$ & $\tau_p$(ms) & $\tau_t$(s) & Rule of $x$ & $\eta(tsa)$ & $\eta(lsa)$ & $pr_{type(x)}$ & $|plan|$ & Goal \\
            \midrule
            G0 & 81 & 4 & 0.02 & 0.11 & \multirow{6}{*}{$x \assign a\cdot t + b$} &  \multirow{6}{*}{$[1,3]$} & \multirow{6}{*}{$[1,1]$}  & $pr_{1st\_poly}$ & 2 & $tsa  \assign  5 \land B_{b}ssa  \assign  4$ \\
            G1 & 1211 & 251 & 0.45 & 1.78 & &   &   & $pr_{1st\_poly}$ & 4 & $tsa  \assign  7 \land B_{b}ssa  \assign  7$ \\
            G2 & 1211 & 223 & 0.70 & 2.42 & &   &   & $pr_{linear\_reg}$ & 4 & $tsa  \assign  7 \land B_{b}ssa  \assign  7$ \\
            G3 & 421 & 109 & 0.16 & 0.66 & &   &   & $pr_{1st\_poly}$ & 4 & $tsa  \assign  7 \land B_{b}ssa  \assign  8$ \\
            G4 & 58708 & 72020 & 88.83 & 196.92 & &   &   & $pr_{1st\_poly}$ & 7 & $B_{c}ssa  \assign  10 \land B_{a}B_{c}ssa  \assign  \none$ \\
            G5 & 64010 & 85696 & 98.70 & 221.81 & &   &   & $pr_{1st\_poly}$ & 7 & $B_{c}ssa  \assign  10 \land B_{a}B_{c}(ssa \neq 10 \land ssa \neq \none)$ \\
            G6 & 17277 & 1221 & 8.84 & 29.32 & $x = a\cdot t^2+b\cdot t+c$& $[1,0,2]$  & $[1,0,0]$  & $pr_{2st\_poly}$ & 6 & $tsa  \assign  38 \land B_{b}ssa  \assign  38$ \\
            G7 & 29 & 305 & 0.03 & 0.13 & $x = a^t$& $[3]$  &  $[1]$ & $pr_{power}$ & 2 & $tsa  \assign  9 \land B_{b}ssa  \assign  9$ \\
            G8 &17277 & 889 & 15.58 & 37.49 & $x=a \cdot sin (b \cdot t + c)$ & $[8,5,4]$  & $ [1,1,1] $& $pr_{sin}$ & 6 & $tsa  \assign  4.23\land B_{b}ssa  \assign  4.23$ \\ 
            \bottomrule
        \end{tabular}
    }
    \caption{Experimental results}
    \label{tab:results}
    \vspace{-4mm}
\end{table*}
\section{Experiment}
\label{sec:experiments}
Experiments are conducted on the most challenging benchmark domain, Grapevine, with the inclusion of dynamically changing variables.
As described in Example~\ref{example:lie}, it involves three agents ($a$, $b$, and $c$) who can \action{share}, \action{lie}, or \action{move} in two connected rooms ($rm_1$ and $rm_2$), and initially, all three agents are located in $rm_1$.
Unlike Example~\ref{example:lie}, in this experiment, the agents are allowed to \action{share} others' secrets (as in PDDL+ Example~\ref{alg:action}).
It is noted that the other agents do not know whether the shared value (e.g., $ssa$) is true or false. 
Therefore, when they share others' secrets, they can only share the value they believe in, which can be equal to $tsa$, $lsa$, or some other values (e.g.,  from a false prediction). 
To explicitly test agents' beliefs rather than knowledge (derived from direct ``observation''), we enforce that the goal cannot be achieved at any time when sharing occurs.

The following outcome metrics are included to show the performance of the PJP model:
the number of generated nodes~($|gen|$), 
the number of segments~($|seg|$),  
the average PJP function used time~($\tau_p$),
the total execution time~($\tau_t$),
the rule of processual variables~(Rule of $x$) and their coefficients ($\eta(tas)$ and $\eta(las)$),
the predictive retrieval functions used~($pr_{type(x)}$),
the plan length~($|plan|$), 
and the goal conditions~(Goal). 
It is noted that $pr_{1st\_poly}$ is defined in Definition~\ref{def:pr_1st_poly}, while the formal definitions and implementation methodologies of all other $pr_{type(x)}$ functions are provided in the supplementary material.
All source code, including domain encoding, predictive retrieval function implementations, and plan validator, has been released as an open-source project in: \textbf{link anonymized for submission},

The experiments are run on a laptop equipped with an Intel\textsuperscript{\textregistered} Core\textsuperscript{\texttrademark} i7-10510U processor and 16 GB of RAM.
The timeout is set to 600 seconds, and the memory limit is 8 GB.
To solely demonstrate the effectiveness of our PJP model, the vanilla version of \emph{Breadth-First Search}~(BFS) is used to avoid influence from the search algorithm.
As mentioned earlier, there is no existing planner tool that can reason about agents' nested beliefs with predictions, so we have no other approach to compare against.

The result is shown in Table~\ref{tab:results}.
Variable $tsa$ and $lsa$ are first-order polynomials in G0-G5, while we also show other example types in G6, G7, and G8.
G0 shows the base case where $a$ only shared $tsa$ once, resulting in $b$ not having enough values\footnote{There are still 4 predictions due to the search frontier (reached) to the depth of 3.} to predict $ssa$ (fall back to static).
G3 shows that $b$ believes $ssa$ is a value that never occurred in the current search path (as $tsa\leq 7$ and $lsa \leq 5$ with all lengths of visited search paths smaller than $5$).
This is caused by $a$ lying at $s_1$ ($ssa=lsa=2$) and sharing at $s_3$ ($ssa=tsa=6$), resulting in $b$ believes $ssa$ is $[0,2,4,6,8]$.
G4 and G5 are challenging instances such as agent $c$ believes the correct $ssa$ value, while agent $a$ is either not aware of that or has false beliefs on $c$ about $ssa$.
%
%
An interesting point that can be noticed is that the choice of the predictive retrieval function may influence the efficiency of the planner.
For example, the $pr_{linear\_reg}$ used in G2 is less efficient (from $\tau_p$) compared to the analytical method in G1.

\subsection{Discussion}
As demonstrated in the previous sections, the PJP model provides modelers with a high degree of flexibility in the design of the predictive retrieval functions. 
It brings a new angle of view to revisit the existing domain.
For example, in the Grapevine domain, a recent lie from someone could result in agents losing a truthful belief about one's secret.
This can be found in Example~\ref{example:lie}, where agent $b$'s belief (using $pr_{1st\_poly}$ in Definition~\ref{def:pr_1st_poly}) about $a$'s secret ($B_b ssa$) was aligned with $tsa$ but it is updated by $a$'s lie in timestamp $6$.
In the presence of lies, incorrect patterns are formed, influencing the prediction of this value both prior to and following the lies. 
A potential solution is an alternative predictive justified perspective function $pr_{dom\_1st\_poly}$ that takes the majority of all observations to reject those outliers.
Specifically, it generates the coefficients from all valid (not $\none$) observed values.
The pair of coefficients with the most occurrences is selected as the fitting parameters to predict the target value.
A slightly complex example is provided in Example~\ref{example:outlier}.

\begin{example}
\label{example:outlier}
Following the problem setting earlier in this section, the objective for agent~$a$ is to have both $b$ and $c$ believe $a$'s true secret, and for $b$ to believe the true secret of $a$. Meanwhile, agent~$c$ aims to sabotage this by lying to $b$ at the end.
Agent~$c$ must first learn the pattern of $sa$, and then lie to $b$ using his own pattern ($lsc = 2t + 4$).
\end{example}


Since the objectives of agents have conflicts, the planner cannot find a centralized optimal plan for ``agent $c$ would like to sabotage this by lying to $b$''.
Therefore, similar to \citeauthor{muise2022}~\shortcite{muise2022}, we use the planner only to validate the plan:
[\action{share}($a,sa,tsa$), \action{stop}, 
\action{share}($a,sa,tsa$), \action{stop},
\action{share}($a,sa,tsa$), \action{stop},
\action{move}($a,rm_2$),
\action{lie}($c,sa,lsc$), \action{stop}].

With the new $P\!R$ function ($pr_{dom\_1st\_poly}$) the above plan makes agent $b$ still hold the correct belief about $ssa$ ($tsa \assign 12 \land B_b ssa \assign 12$).
The linear regression ($pr_{linear\_reg}$) results in $B_b ssa  \assign  20.24$, while the default predictive retrieval function ($pr_{1st\_poly}$) in Definition~\ref{def:pr_1st_poly} results in $B_b sa  \assign  20$.
With the flexibility of the PJP model, other commonly used outlier rejection algorithms such as \emph{three-sigma rule} and \emph{Random Sample Consensus} (RANSAC)~\cite{FISCHLER1987726} can also be implemented as an external function.
This example demonstrates the possibilities to utilize the PJP model incorporating external functions to solve broader problems.

\section{Conclusion \& Future Work}
\label{sec:conclusion}
In conclusion, to fill the gap between real-world application and the ``static environment'' assumption in existing EP studies, we proposed the PJP model by introducing the processual variables and predictive retrieval functions.
Following the same structure as the JP model, our approach retains its advantages of being action-model-free (not requiring explicit epistemic effects in actions) and capable of arbitrary nesting of beliefs.
The experiments demonstrate the modeler has the flexibility to make agents generate beliefs with reasonable predictions by incorporating different external functions, which shows the PJP model's potential to be applied in various applications.
However, the proposed model still requires the pre-defined processual variable model, constraining its applications where consensus is impossible to be reached.
A potential solution is to introduce the learning-based methods to eliminate the predefined rules, improving the model's adaptability.

\appendix
\section{Semantics and Axiomatic Validity}
\label{sec:validity}
By constructing predictive justified perspectives (Definition~\ref{def:pjpf}), the reasoning of the epistemic formulae in $L_{K\!B}$ follows the ternary semantics defined in the JP model (Definition~\ref{def:jp:ternary}), except that the perspective function is changed from the JP function (Definition~\ref{def:jpf}) to the PJP function (Definition~\ref{def:pjpf}).

The complexity of the ternary semantics is in polynomial, if $O$ and $P\!R$ are in polynomial.
Specifically, given a ternary function instance $T[M,\seq,\varphi]$, the worst-case scenario is that $\varphi$ is a belief formula with a depth of $d$. 
Then, according to its definition, the complexity class is $\Theta(|V| \cdot d \cdot |\seq|) \cdot \Theta (o) \cdot \Theta (pr)$, where $o$ and $pr$ are the most complex observation function $(o \in \{O_1, \dots, O_{|Agt|}\})$ and predictive retrieval function $(pr \in P\!R)$ respectively. 
That is, the depth of the epistemic formula, which usually causes exponential blowup in DEL or Knowledge-Base approaches, would not cause exponential blowup in the PJP model (similar to the JP model).

To examine the soundness of the PJP model, 
we show that our ternary semantics follows the \textbf{KD45} axiomatic system for all epistemic logics related to belief.
At the end of this section, we also discuss a $K\!B$ axiom that is affected by the property of the predictive retrieval function. 

Firstly, we show the agents' predictive justified perspectives are consistent with their own observations (Theorem~\ref{thm:jpcontaino}).
\begin{theorem}
    \label{thm:jpcontaino}
    Given a state sequence $\seq$, let $i$ be an agent from $Agt$, for any timestamp $t$, we have:
    \[
        \observation_i(\seq[t]) \subseteq \f_i(\seq)[t]
    \]
\end{theorem}
\begin{proof}
For any $v$ in $V$ such that $v\!\!\in\!\! \observation_i(\seq[t])$, according to Line (1) in Definition~\ref{def:pjpf} ($e \assign pr_{type(v)}(\observation_i(\seq),t,v)$) and Preserving Consistency in Definition~\ref{def:pr_function}, the retrieved predictive value of $v$ equals $\observation_i(\seq[t])(v)$.
Then, following $v \!\in\! \observation_i(\seq[t])$, the generated $s_t''$ contains the assignment $v \assign \observation_i(\seq[t])(v)$.
In addition, $s_t''$ overriding the none state $s_{\none}$ has no effects on existing assignments in $s_t''$, which means assignment $v \assign \observation_i(\seq[t])(v)$ is in $\f_i(\seq)[t]$.
Hence, $\observation_i(\seq[t])(v) = f_i(\seq)[t](v)$, which concludes $\forall v \in V,\ v \in \observation_i(\seq[t]) \Rightarrow \observation_i(\seq[t])(v) = f_i(\seq)[t](v)$.
Therefore, Theorem~\ref{thm:jpcontaino} holds.
\end{proof}

Intuitively, agent's predictive justified perspective $f_i(\seq)$ of the given state sequence $\seq$ should be the same as their prediction $f_i(f_i(\seq))$ of what they predicted ($f_i(\seq)$).
Thus, we propose the following theorem.
\begin{theorem}[Idempotence of Predictive Retrieval Function]
\label{thm:fififi}
A predictive justified perspective function $f_i$ is idempotent:
\[
f_i(f_i(\seq)) = f_i(\seq)
\]
\end{theorem}

\begin{proof}
Let the input sequence be $\seq = [s_0, \dots, s_n]$.
First, we can construct a prediction $\vec{p} = [p_0, \dots, p_n]$ based on agent $i$'s observation $\observation_i(\seq)=[\observation_i(\seq[0]),\dots,\observation_i(\seq[n])]$ of the given input state sequence $\seq$ following the same process in Definition~\ref{def:pr_function}, where for $t' \in [0,n]$, $p_t = \{ v' \assign e' \mid e' = pr_{type(v')}(\observation_i(\seq),t',v')\}$.

According to the Preserving Consistency (Definition~\ref{def:pr_function}) of all predictive retrieval functions, we have for all $t < |\seq|$, $\observation_i(\seq[t]) \subseteq \svec{p}[t]$.
Following Line (1) and Line (2) in Definition~\ref{def:pjpf} during the construction of $f_i(\seq)$, for all $t < |\seq|$, since $s''_t = \{v \assign e \mid v \!\in \!\observation_i(s_t) \lor v\! \notin\! \observation_i(s_t\langle \{v \assign e\} \rangle) \}$ ($e = pr_{type(v)}(\observation_i([s_0,\dots,s_n]),t,v)$), we have $s''_t \subseteq \svec{p}[t]$.
Then, Line (3) of the same definition fills the missing variables (of $s_t''$) with none value assignments.
It is noted that those added missing variables by Line (3) are not in agent $i$'s observation, which means $\observation_i(s_t') = \observation_i(s_t'')$.
According to the Contraction property of the observation function (in Definition~\ref{def:observation}), we have $\observation_i(s_t') \subseteq s''_t \subseteq \svec{p}[t]$.
From Theorem~\ref{thm:jpcontaino}, we have $\observation_i(s_t) \subseteq s_t'$.
According to Idempotence and Monotonicity of the observation function, $\observation_i(s_t) \subseteq \observation_i(s_t')$.

Since $f_i(\seq)$ is formed by $[s_0',\dots,s_n']$, the input for each variable's predication is $\observation_i(f_i(\seq))$ when generating $f_i(f_i(\seq))$.
In addition, $\observation_i(f_i(\seq))$ is effectively $[\observation_i(s_0'),\dots, \observation_i(s_n')]$.
According to the paragraph above, for all timestamps $t' \! < \! |\seq|$, $\observation_i(s_t) \! \subseteq \! \observation_i(s_t') \! \subseteq \! \svec{p}[t]$.
Due to the Recursive Consistency of all predictive retrieval functions, we have $\forall t \! < \! |\seq|, \forall v \! \in\! V \Rightarrow pr_{type(v)}(\observation_i(f_i(\seq)),t,v)$ $ \!\assign \!  pr_{type(v)}(\vec{p},t,v) \! \assign\! pr_{type(v)}(\observation_i(\seq),t,v)$.
Comparing the construction of $f_i(\seq)$ and $f_i(f_i(\seq))$ from Definition~\ref{def:pjpf}, since Line (1) retrieves the same value, Line (2) and Line (3) will have the same effects. 
Therefore, Theorem~\ref{thm:fififi} holds.
\end{proof}

\vspace{-2mm}
With Theorem~\ref{thm:fififi}, we can now show that the ternary semantics $T$ (Definition~\ref{def:jp:ternary}) in the PJP model satisfies the axioms of the \textbf{KD45} system.
\begin{theorem}
    The ternary semantics for the belief operator $B$ follows the axiomatic system \textbf{KD45} for any agent $i\in Agt$: 
    \vspace{2mm}
    
    \begin{tabular}{@{~}l@{~}l}
    	\textbf{K} & (Distribution):           $B_i \varphi \land B_i(\varphi \Rightarrow \psi) \Rightarrow B_i \psi $\\[1mm]
    	\textbf{D} & (Consistency):            $B_i \varphi \Rightarrow \neg B_i \neg \varphi $\\[1mm]
    	\textbf{4} & (Positive Introspection): $B_i \varphi \Rightarrow  B_i B_i \varphi $\\[1mm]
    	\textbf{5} & (Negative Introspection): $\neg B_i \varphi \Rightarrow  B_i \neg B_i \varphi $\\
    \end{tabular} 
\end{theorem}
\begin{proof}
    For Axiom \textbf{K}, $T[\seq, B_i \varphi \land B_i(\varphi \Rightarrow \psi)] \assign 1$ is equivalent to $\min(T[\seq, B_i \varphi],T[\seq, B_i (\varphi \Rightarrow \psi)]) \assign 1$, which indicates both $T[f_i(\seq), \varphi]$ and $T[f_i(\seq), \varphi \Rightarrow \psi]$ are $1$.
    Therefore, we have $T[f_i(\seq), \psi] \assign 1$, which makes $T[\seq, B_i \psi] \assign 1$.

    For Axiom \textbf{D},  $T[\seq, B_i \varphi] \assign 1$ means $T[f_i(\seq), \varphi] \assign 1$.
    $T[\seq, \neg B_i \neg \varphi] \assign 1-T[\seq, B_i \neg \varphi]$.
    Since $T[\seq, B_i \neg \varphi] \assign T[f_i(\seq), \neg \varphi] \assign 1 - T[f_i(\seq), \varphi]$, we have $T[\seq, B_i \neg \varphi] \assign 0$, which means  $T[\seq, \neg B_i \neg \varphi] \assign 1$.

     For Axiom \textbf{4}, the premise $T[\seq, B_i  \varphi] \assign 1$ means $T[f_i(\seq), \varphi] \assign 1$.
     By Theorem~\ref{thm:fififi}, $f_i(f_i(\seq))  \assign  f_i(\seq)$
     Therefore, $T[f_i(\seq), \varphi] \assign  T[f_i (f_i (\seq)), \varphi] \assign 1$, which means that $T[f_i(\seq), B_i  \varphi] \assign T[\seq, B_i B_i  \varphi] \assign 1$.
     Thus, Axiom \textbf{4} holds.

     Axiom \textbf{5} can be proved in a similar way as \textbf{D} and \textbf{4}.
\end{proof}

In addition to \textbf{KD45}, the PJP model straightforwardly satisfies the $K\!B$ axiom \textbf{KB1} ($K_i \varphi \Rightarrow B_i \varphi$) and \textbf{KB2} ($B_i \varphi \Rightarrow K_i B_i \varphi$), based on Theorem~\ref{thm:jpcontaino}.
In many epistemic logic systems, Axiom \textbf{D}, Axiom \textbf{5}, and Axiom \textbf{KB1} cannot be true at the same time due to the ``unwanted'' axiom ($B_i K_i \varphi \Rightarrow K_i \varphi$); however, as \citeauthor{guang2025thesis}~\shortcite{guang2025thesis} claimed, in \textbf{justified} knowledge and belief systems (e.g. the JP model), the ``unwanted'' axiom is valid because of their strict definition of knowledge and belief. 
Similarly, we discuss whether the ``unwanted'' axiom holds in our PJP model.

Given $M  \! \assign \! (Agt, \Omega,\mathbb{D}, \pi, O_1, \ldots, O_{|Agt|}, P\!R)$, a PJP model instance, where all $pr_{type(\cdot)} \! \in \! P\!R$ satisfy the preserving consistency and recursive consistency, the ``unwanted'' axiom is not necessarily guaranteed to hold.
This is because what agent $i$ sees in $i$'s predictive justified perspective ($f_i(\seq)$) could be more than what agent $i$ sees in the given state sequences $\seq$.
The original observation of agent $i$ ($\observation_i(\seq)$) can be seen by $i$ from the generated $i$'s predictive justified perspective. 
That is, $\forall t < |\seq| \Rightarrow \observation_i(\seq[t]) \subseteq \observation_i(f_i(\seq)[t])$.
However, the predicted values of those unobserved variables might allow agent $i$ to see more in $i$'s own perspective.
In other words, $\forall t < |\seq| \not\Rightarrow \observation_i(f_i(\seq)[t]) \subseteq \observation_i(\seq[t])$.
This results in what agent $i$ believes that $i$ knows could be inconsistent with what $i$ actually knows.

However, if all  $pr_{type(\cdot)} \in P\!R$ also satisfy the optional reconstructive consistency, the ``unwanted'' axiom can be guaranteed to hold.
\begin{theorem}
    \label{thm:bkisk}
    Given $M  \! \assign \! (Agt, \Omega,\mathbb{D}, \pi, O_1, \ldots, O_{|Agt|}, P\!R)$, a PJP Model instance, where all $pr_{type(\cdot)} \in P\!R$ satisfy preserving consistency, recursive consistency, and reconstructive consistency, the following theorem holds:
    \[
        \forall \varphi, B_i K_i \varphi \Rightarrow K_i \varphi
    \]
\end{theorem}

\begin{proof}
    The premise of the above implication $T[\seq, B_i K_i \varphi] \! \assign \!1$ means $T[f_i(\seq), K_i \varphi] \! \assign \!1$, which is $T[f_i(\seq), \varphi] \! \assign \!1$ and $T[f_i(\seq),S_i \varphi] \! \assign \!1$, while, the conclusion of the theorem implication $T[\seq,  K_i \varphi]$ is effectively the same as $\min(T[\seq, \varphi], T[\seq, S_i \varphi])$ according to our semantics~(Definition~\ref{def:jp:ternary}).

    For the latter part ($T[\seq, S_i \varphi]$), when constructing $\f_i(\seq)$, the input sequence for all predictive retrieval calls is $\observation_i(\seq)$ according to Line (1) of Definition~\ref{def:pjpf}.
    Since all $pr_{type(\cdot)} \in P\!R$ satisfy the reconstructive consistent (in Definition~\ref{def:pr_function}), we have $\observation_i(\seq) = \observation_i(f_i(\seq))$.
    Since $T[f_i(\seq),S_i \varphi]=1$ means $T[\observation_i(f_i(\seq)), \varphi]$ is evaluated as either $0$ or $1$, we have $T[\observation_i(\seq), \varphi]$ is evaluated as either $0$ or $1$, which concludes $T[\seq,S_i \varphi]=1$.

    Then, according to $T[\observation_i(\seq), \varphi] \neq \unknown$ (item a in Definition~\ref{def:jp:ternary}), a partial-state sequence $\observation_i(\seq)$ is sufficient to evaluate $\varphi$. 
    It means the predicted values of unobserved variables are irrelevant to $\varphi$.
    In addition, since $\observation_i(\seq) = \observation_i(f_i(\seq))$, the partial-state sequences that determine the value of $\varphi$ from $f_i(\seq)$ and $\seq$ are the same.
    Thus, $T[f_i(\seq), \varphi]=1$ gives $T[\seq, \varphi]=1$.
    
    Therefore, Theorem~\ref{thm:bkisk} holds.    
\end{proof}

\section{Other Common Predictive Retrieval Functions}
\label{app:pr_functions}

This appendix provides formal definitions of the predictive retrieval functions utilized in the main text.

\addtocounter{mydefinition}{-1}
\addtocounter{subdefinition}{0}
\begin{subdefinition}[Linear Regression Predictive Retrieval Function]
\label{def:pr_linear_regression}
Given AT $\assign \{ j\! \mid \!v\! \in \!\seq[j] \land \seq[j](v) \neq \none \}$, the predictive retrieval function by linear regression,  $pr_{linear\_reg}(\vec{s}, t,v)$, is defined as:
\[
pr_{linear\_reg}(\vec{s}, t,v) =
\begin{cases}
\none & \text{if } |\text{AT}| = 0, \\
\seq[\max(\text{AT})](v) & \text{if }  |\text{AT}| < 2, \\
\vec{w}^\top \vec{x}_t + b & \text{if } |\text{AT}| \geq 2,
\end{cases}
\]
where $\vec{x}_t = [1, t]$, and the regression parameters $\vec{w}$ and $b$ are computed based on all observed timestamps $\{t_j \mid j \in \text{AT}\}$ and corresponding values $\{e_j = s_{t_j}(v)\}$.

The parameters $\vec{w}$ and $b$ are learned by minimizing the mean squared error (MSE) over all observed values~\cite{sklearn_api}:
\[
\text{MSE} = \frac{1}{|\text{AT}|} \sum_{j \in \text{AT}} \left( \vec{w}^\top \vec{x}_{t_j} + b - e_j \right)^2,
\]
where:
$e_j = s_{t_j}(v)$ are the observed values.

For prediction at timestamp $t$, $\vec{x}_t = [1, t]$ is used to compute:
\[\hat{e}_t = \vec{w}^\top \vec{x}_t + b. \defmathend
\]
\end{subdefinition}

\begin{subdefinition}[Second-order Polynomial Predictive Retrieval Function]
\label{def:pr_2nd_poly}
    Given $\text{AT} = \{ j\! \mid \!v\! \in \!\seq[j] \land \seq[j](v) \neq \none \}$,
            $\text{LT}\! =\! \{j \!\mid\! j\!\in\! \text{AT} \!\land j\! \leq\! t\}$, and $\text{RT}\! =\! \{j\! \mid \!j\in\! \text{AT}\! \land\! j\! >\! t\}$, the predictive retrieval function for the second-order polynomial, \( pr_{2nd\_poly}(\vec{s}, t,v) \), is defined as~\cite{Harris_2020}:
    \[
    pr_{2nd\_poly}(\vec{s}, t,v) =
    \begin{aligned}
    \begin{cases} 
    \none & \text{if } |\text{AT}| = 0 \\
    \seq[\max(\text{AT})](v) & \text{if } |\text{AT}|< 3 \\
    a t^2 + b t + c & \text{if } |\text{AT}| \geq 3
    \end{cases}
    \end{aligned}
    \]
    where:
    \begin{equation*}
        \begin{aligned}
            t_1, t_2, t_3 &= \mathrm{Top}_3(\text{AT}) \text{such that } t_1 < t_2 < t_3, \\
            e_1 &= \seq[t_1](v), \quad e_2 = \seq[t_2](v), \quad e_3 = \seq[t_3](v), \\
            \{ a, b, c \} &= \text{polyfit}([t_1, t_2, t_3], [e_1, e_2, e_3], 2)
        \end{aligned}\defmathend
    \end{equation*} 
\end{subdefinition}

\begin{subdefinition}[Power Function Predictive Retrieval Function]
\label{def:pr_power}
Given $\text{AT} = \{ j\! \mid \!v\! \in \!\seq[j] \land \seq[j](v) \neq \none \}$, the predictive retrieval function for the power function, \( pr_{power}( \vec{s}, t,v) \), is defined as:
\[
pr_{power}(\vec{s}, t, v) = 
\begin{cases} 
\none & \text{if } |\text{AT}| = 0, \\
a^t & \text{if } |\text{AT}| > 0 \land |a| =1 \\
\seq[\max(\text{AT})](v) & \text{otherwise}\\
\end{cases}
\]
where:
\begin{equation*}
    \begin{aligned}
        t_{\text{last}} &= \max(\text{AT}), \quad y_{\text{last}} = \seq[t_{\text{last}}](v), \\
        a &= 
        \begin{cases} 
        \{y_{\text{last}}^{1 / t_{\text{last}}}\} & \text{if } t_{\text{last}} \text{ is odd}, \\
        \{y_{\text{last}}^{1 / t_{\text{last}}}, -y_{\text{last}}^{1 / t_{\text{last}}}\} & \text{if } t_{\text{last}} \text{ is even and } y_{\text{last}} > 0, \\
        \emptyset & \text{otherwise}.
        \end{cases}
    \end{aligned}
\end{equation*}
\defnormalend
\end{subdefinition}

\begin{subdefinition}[Sinusoidal Predictive Retrieval Function]
\label{def:pr_sin}
Given $\text{AT} = \{ j\! \mid \!v\! \in \!\seq[j] \land \seq[j](v) \neq \none \}$, the predictive retrieval function for the sinusoidal model, \( pr_{sin}(\vec{s}, t,v) \), is defined as:
\[
pr_{sin}(\vec{s}, t,v) = 
\begin{cases} 
\none & \text{if } |\text{AT}| = 0, \\
a \cdot \sin\left(b \cdot \frac{\pi}{2} \cdot t + c \right) & \text{if } |\text{AT}| \geq 3, \\
\seq[\max(\text{AT})](v) & \text{otherwise}.
\end{cases}
\]
where:
\begin{equation*}
    \begin{aligned}
        x_{\text{obs}} &= [t_1, t_2, t_3, \ldots],\\
        y_{\text{obs}} &= [\seq[t_1](v), \seq[t_2](v), \seq[t_3](v), \ldots], \\
        \{a, b, c\} &= 
        \arg\min_{a, b, c} \sum_{i=1}^{|\text{AT}|} \left( y_{\text{obs}, i} - a \cdot \sin\left(b \cdot \frac{\pi}{2} \cdot x_{\text{obs}, i} + c \right) \right)^2\\
    \end{aligned}
\end{equation*}
Here, the coefficients \(a\), \(b\), and \(c\) are obtained using nonlinear least squares fitting on the observed points \((x_{\text{obs}}, y_{\text{obs}})\)~\cite{2020SciPy-NMeth}.\defnormalend 
\end{subdefinition}

\begin{subdefinition}[First-order Polynomial Predictive Retrieval Function (Dominant Coefficients)]
\label{def:pr_1st_poly_Dominant_coef}
    Given $\text{AT} = \{ j\! \mid \!v\! \in \!\seq[j] \land \seq[j](v) \neq \none \}$, the predictive retrieval function for the first-order polynomial in the plan validation section, \( pr_{dom\_1st\_poly}(\vec{s},t,v) \), can be defined as:
    \[
    pr_{dom\_1st\_poly}(\vec{s},t,v) = 
    \begin{cases} 
    \none & \text{if } |\text{AT}| = 0, \\
    a \cdot t + b & \text{if } |\text{AT}| \geq 2, \\
    \seq[\max(\text{AT})](v) & \text{otherwise}.
    \end{cases}
    \]
    where:
    \begin{equation*}
        \begin{aligned}
            x_{\text{obs}} &= [t_1, t_2, \ldots], \quad y_{\text{obs}} = [\seq[t_1](v), \seq[t_2](v), \ldots], \\
            \{a, b\} &= \arg\max_{a, b} \text{Frequency}(a, b),
        \end{aligned}
    \end{equation*}
    with Frequency(\(a, b\)) is the number of times the combination \(a\) and \(b\) appears across all point pairs.\defnormalend
\end{subdefinition}

\bibliography{pjp}


\end{document}